\documentclass[conference]{IEEEtran}
\usepackage{times}

\usepackage[numbers]{natbib}
\usepackage{multicol}
\usepackage[bookmarks=true]{hyperref}

\usepackage[utf8]{inputenc} 
\usepackage[T1]{fontenc}    
\usepackage{hyperref}       
\usepackage{url}            
\usepackage{booktabs}       
\usepackage{amsfonts}       
\usepackage{nicefrac}       
\usepackage{microtype}      
\usepackage{subcaption}     

\usepackage{color}
\usepackage{amsmath, amsthm, amsfonts}
\usepackage{appendix}

\usepackage{caption}
\newcommand{\R}{\mathbb{R}}

\newcommand{\I}{\mathbf{I}}

\newtheorem{thm}{Theorem} 
\newtheorem{prop}[thm]{Proposition}

\newcommand{\states}{S}
\newcommand{\actions}{A}
\newcommand{\state}{{s}}
\newcommand{\action}{{a}}
\newcommand{\newState}{{s'}}
\newcommand{\rewardFn}{r}
\newcommand{\transProb}{P_{\state,\action}}
\newcommand{\transProbs}{\{P_{\state,\action}\}}
\newcommand{\rewardParam}{\theta}
\newcommand{\rewardParams}{\Theta}
\newcommand{\prior}{p}

\newcommand{\planner}{d} 
\newcommand{\plannerParamed}[1]{\planner(#1)}
\newcommand{\policy}{\pi}
\newcommand{\policySpace}{\Pi}

\newcommand{\demonstrator}{human}

\newcommand{\inferer}{robot}

\newcommand{\controls}{a}

\newboolean{include-notes}
\setboolean{include-notes}{true}
\newcommand{\adnote}[1]{\ifthenelse{\boolean{include-notes}}
 {{\color{red}AD: #1}}{}}
\newcommand{\lcnote}[1]{\ifthenelse{\boolean{include-notes}}
 {{\color{green}LC: #1}}{}}
\newcommand{\acnote}[1]{\ifthenelse{\boolean{include-notes}}
 {{\color{blue}AC: #1}}{}}
 \newcommand{\dsnote}[1]{\ifthenelse{\boolean{include-notes}}
 {{\color{purple}DS: #1}}{}}
 
\usepackage{graphicx}       

\newcommand{\figref}[1]{Fig. \ref{#1}}
\newcommand{\secref}[1]{Sec. \ref{#1}}

\title{Human irrationality: both bad\\ and good for reward inference}

\begin{document}


\author{\authorblockN{Lawrence Chan}
\authorblockA{University of California, Berkeley\\
chanlaw@berkeley.edu}
\and
\authorblockN{Andrew Critch}
\authorblockA{University of California, Berkeley\\
critch@berkeley.edu}
\and
\authorblockN{Anca Dragan}
\authorblockA{University of California, Berkeley\\
anca@berkeley.edu}}


%

\maketitle
\begin{abstract}
Assuming humans are (approximately) rational enables robots to infer reward functions by observing human behavior. But people exhibit a wide array of irrationalities, and our goal with this work is to better understand the effect they can have on reward inference. The challenge with studying this effect is that there are many types of irrationality, with varying degrees of mathematical formalization. We thus operationalize irrationality in the language of MDPs, by altering the Bellman optimality equation, and use this framework to study how these alterations would affect inference.  

We find that wrongly modeling a systematically \emph{irrational} human as noisy-\emph{rational} performs a lot worse than correctly capturing these biases -- so much so that it can be better to skip inference altogether and stick to the prior! More importantly, we show that an \emph{irrational} human, when \emph{correctly} modelled, can communicate \emph{more} information about the reward than a \emph{perfectly rational} human can. That is, if a robot has the correct model of a human's irrationality, it can make an even \emph{stronger} inference than it ever could if the human were rational. Irrationality fundamentally helps rather than hinder reward inference, but it needs to be correctly accounted for.
We explain this theoretically by showing that irrationality can increase the mutual information between the policy and the reward. We also replicate the result in an autonomous driving domain with myopic/greedy demonstrations.

\end{abstract}

\section{Introduction}
\label{intro}

Motivated by difficulty in reward specification \citep{lehman2018surprising}, inverse reinforcement learning (IRL) methods estimate a reward function from human demonstrations
\citep{ng2000algorithms,abbeel2004apprenticeship,kalman1964linear,jameson1973inverse,mombaur2010human}. 
The central assumption behind these methods is that human behavior is {rational, i.e., optimal with respect to their reward (cumulative, in expectation)}. 
Unfortunately, decades of research in behavioral economics and cognitive science \citep{Chipman2014Cognitive} have unearthed a deluge of \emph{irrationalities}, i.e., of ways in which people deviate from optimal decision making: hyperbolic discounting, scope insensitivity, illusion of control, decision noise, loss aversion, to name a few. 
As a community, we are starting to account for some of these irrationalities in different ways \citep{ziebart2008maximum, ziebart2010modeling, singh2017risk, reddy2018you, evans2016learning,shah2019feasibility}. In this work, we seek to understand how the \emph{"ground truth"} human irrationality (type and amount) pairs with the robot's \emph{model} of the human to affect the \emph{performance} of reward inference. 

On the human \emph{model} side, prior work has shown that a more accurate model of the ground-truth human behavior will lead to better reward learning performance. However, what we still do not understand is: how big is the difference between reward learning performance for an irrational human when the human is correctly modelled (as irrational), as opposed to incorrectly modelled (as rational), for different types of irrationality? A very large difference would mean we should keep pushing towards better human models.

On the \emph{ground truth} human side, we may intuit that a perfectly rational human would be the ideal \demonstrator{} for reward inference and that a robot using a rational model will perform the best reward inference, i.e. that the (rational human, rational model) pair is ideal. Even if this proves true, it still raises questions about how other human-model pairs fare. For example, just how hard does human irrationality make reward learning, assuming the robot accurately models the human? While this might seem hypothetical because we cannot dictate what kind of human interacts with the robot, the answer is nonetheless important in order to understand the fundamental limits of reward learning. Namely, what are the effects of irrationality on the upper bounds on performance, if the \inferer{} correctly models the \demonstrator{}? If some irrationality types lead to inherently ambiguous behavior, such that no \inferer{} could perform well, this is something the reward learning community needs to be aware of.

\begin{figure}[t]
    \centering

    \includegraphics[width=1\columnwidth]{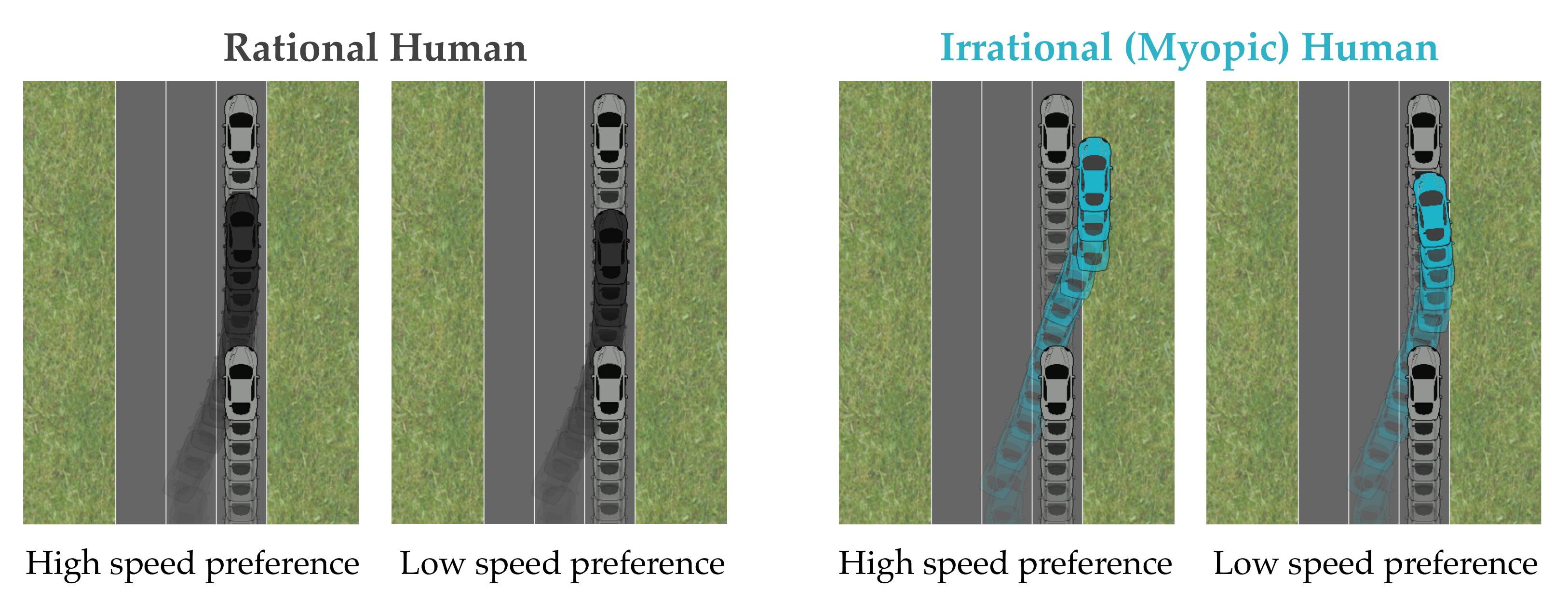}
    \caption{On the left, the rational human (dark gray) behaves in the same way regardless of their speed preference. On the right, the myopic (short horizon) human (cyan) behaves \emph{differently} depending on their speed preference: merging and staying in lane vs. myopically starting to pass on the right to make more progress. Myopia induces this diversity of behavior as a function of the human's reward, increasing the amount of information in the behavior. This makes it fundamentally easier for a \inferer{} to differentiate what the reward is. }
    \label{fig:front}
\end{figure}

One challenge with conducting such an analysis is that there are many irrationalities in the psychology and behavioral economics literature, with varying degrees of mathematical formalization versus empirical description. To structure the space for our analysis, we operationalize irrationalities in the language of MDPs by systematically enumerating possible deviations from the Bellman equation -- imperfect maximization, deviations from the true transition function, etc. This gives us a formal framework in which we can simulate irrational behavior, run reward inference, and study its performance. 
Armed with this formalism, we then explore the various impacts of irrationality on reward learning in three families of environments: small random MDPs, a more legible gridworld MDP, and an autonomous driving domain drawn from the literature \citep{sadigh2016planning}. 

\subsection{Irrationality can make human behavior \emph{more} informative.} Our most surprising finding comes from analyzing the question of how "ground truth" human irrationality affects reward inference, when the correct model is known to the robot. We expected that perfectly rational behavior is more informative than suboptimal, biased behavior. Instead, we found that certain irrationalities actually \emph{improve} the quality of reward inference. When the \demonstrator{} exhibits certain biases, correctly modeling these biases can help make stronger inferences -- not only compared to incorrect modeling, but even compared to the best-case performance when learning from a perfectly rational \demonstrator{}! That is, in all three environments: an irrational \demonstrator{} correctly modelled as irrational by the \inferer{} leads to better reward inference than a rational \demonstrator{} correctly modelled as rational,
for some types and amounts of irrationality.


We explain this theoretically from the perspective of the \emph{mutual information} between the \demonstrator{} behavior and the reward parameters, proving that some irrationalities are arbitrarily more informative than rational behavior. An intuitive example is in \figref{fig:front}, where myopic behavior varies with the reward, but rational behavior does not -- this makes it fundamentally possible for the \inferer{} to figure out the reward from the myopic behavior, but not from the rational one. We prove that there are cases where the rational behavior is not informative at all, whereas some irrational behavior achieves the information-theoretic upper bound; we also prove that the popular Boltzmann-rationality model is arbitrarily more informative than perfect rationality. 
While knowing the human irrationality model is hard, this result is overall very encouraging: it is probably easier to learn irrationality models of people, than to make people act rationally -- so in the end, it is a good thing that irrational behavior is the more informative one. 

\begin{figure*}
    \centering
    \vspace{-0.5cm}
    \includegraphics[width=0.6\textwidth]{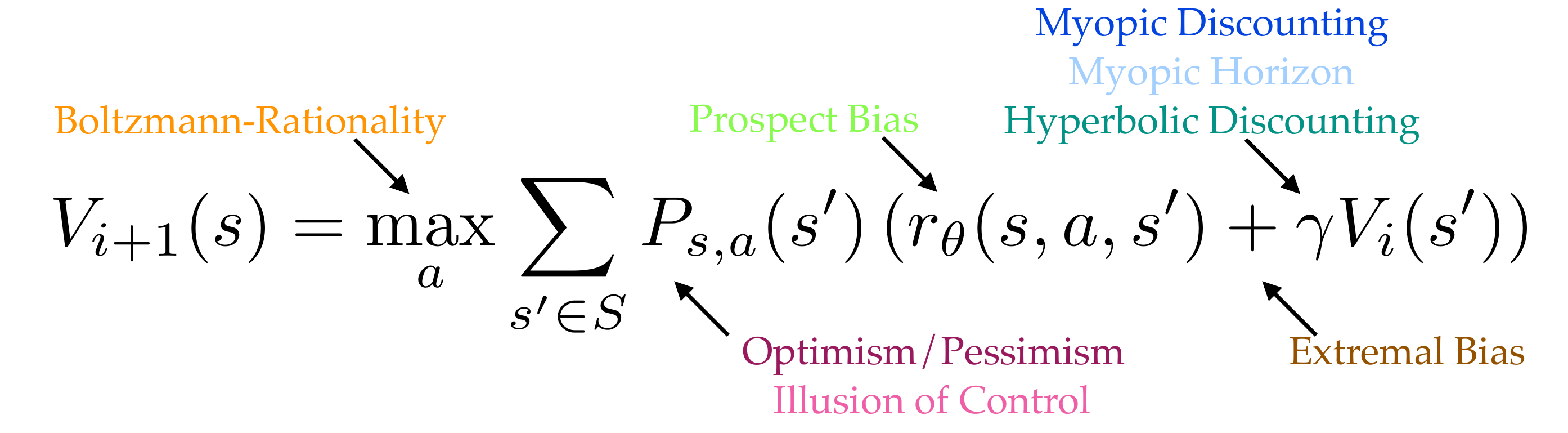}
    \caption{In \secref{irrationality-types}, we modify the components of the Bellman update to cover different types of irrationalities: changing the max into a softmax to capture noise (Boltzmann), changing the transition function to capture optimism/pessimism or the illusion of control, changing the reward values to capture the nonlinear perception of gains and losses (Prospect), changing the average reward over time into a maximum (Extremal), and changing the discounting to capture more myopic decision-making. }
    \label{fig:biases}
\end{figure*}

\subsection{Unmodelled irrationality leads to remarkably poor reward inference.} We continue with an analysis of what the \emph{model} that the robot uses should be, if the human is in fact irrational.
It might seem that we can't immediately benefit from the knowledge that irrationalities help inference unless we have a comprehensive understanding of human decision-making, and so we should just stick to the status quo of modeling people as rational. However, prior work \cite{shah2019feasibility, singh2017risk} has already shown benefits of algorithms that approximately capture irrationality. What we contribute is the finding that modeling irrational \demonstrator{}s as (noisily-)rational (instead of accounting for their biases) can lead to worse outcomes than not performing inference at all and just using the prior (\secref{misspec-effects}). Encouragingly, we also find evidence that even just modeling the \demonstrator{}'s irrationality only  \emph{approximately}, i.e. getting the type of irrationality correct but not the exact parameters, allows a \inferer{} to outperform modeling the \demonstrator{} as noisily-rational. This reinforces prior works that studied specific algorithms for capturing specific biases with a more systematic analysis (\secref{additional-misspec}). 


\subsection{Contributions}

Overall, we contribute 1) a theoretical and empirical analysis of the effects of different irrationalities on reward inference,  2) a way to systematically formalize and cover the space of irrationalities in order to conduct such an analysis, and 3) evidence and explanation for the importance and benefit of accounting for irrationality  during inference. The implications are that we should probably stop assuming rationality (because it can lead to worse inference than sticking to the prior), and further prior work on learning models of irrationality (because there is a lot of information to be exploited from irrational behavior). While basic intuition says that it is good to model irrationality because people are irrational, what we argue in this paper is that there is an additional reason for modeling irrationality: the way people are irrational communicates additional information about the reward, and correct modeling will enable robots to tap into that.

\section{Framework: Biases as deviations from the Bellman update}
\label{framework}
\subsection{Exploring biases through simulation}
While ideally we would recruit human subjects with different irrationalities and measure how well we can learn rewards, this is prohibitive because we do not get to dictate someone's irrationality type: people exhibit a mix of them, some yet to be discovered. Further, measuring the accuracy of inference from observing real humans is complicated by the fact that we lack ground truth access to the human's reward function.  For instance, suppose we asked subjects to produce a set of (behavior, reward function) pairs. We could then try to predict the reward functions from the behaviors.  But how did we, the experimenters, infer the reward functions from the people? If we are wrong in our assumptions about which irrationalities are affecting their behavior and/or explicit reports of rewards, we would remain deluded about the subjects' true preferences.

To address these issues, we \emph{simulate} \demonstrator{} behavior subject to different irrationalities, run reward inference, and measure the performance against the ground truth, i.e., the accuracy of a Bayesian posterior on the reward parameter given the (simulated) \demonstrator{}'s inputs.
\subsection{Background and formalism}
Consider an Uncertain-Reward MDP (URMDP) \citep{bagnell2001solving, regan2011robust, desai2017uncertain} $\mathcal M = (\states, \actions, \transProbs, \gamma, \rewardParams, \prior,  \rewardFn )$, consisting of finite state and action sets $\states$ and $\actions$, distributions over states $\transProbs$ representing the result of taking action $a$ in state $s$, discount rate $\gamma \in [0, 1)$, a (finite) set of reward parameters $\rewardParams$, a prior distribution $\prior \in \Delta(\states \times \rewardParams)$ over starting states and reward parameters, and a parameterized state-action reward function $\rewardFn: \rewardParams \times \states \times \actions \times \states \rightarrow \R$, where $\rewardFn_{\rewardParam}(\state, \action, \newState)$ represents the reward received. 

We assume that the \demonstrator's policy $\policy$ satisfies $\policy = \plannerParamed{\theta}$, where $\planner$ is an (environment-specific) \textbf{planner} $\planner: \rewardParams \rightarrow \policySpace$ that returns a (possibly stochastic) policy given a particular reward parameter $\theta$. The \textbf{rational} \demonstrator{} uses a planner $\planner_{\textrm{Rational}}$ that, given a reward parameter $\theta$, returns a policy that maximizes its expected value. On the other hand, we say that an \demonstrator{} is \textbf{irrational} if its planner returns policies with lower expected value than the optimal policy, for at least one $\rewardParam \in \rewardParams$.

\subsection{Types and degrees of irrationality}\label{irrationality-types}

There are many possible irrationalities that people exhibit \citep{Chipman2014Cognitive}, far more than what we could study in one paper. 
To provide good coverage of this space, we start from the Bellman update, and systematically manipulate its terms and operators to produce a variety of different irrationalities that deviate from the optimal MDP policy in complementary ways. For instance, operating on the discount factor can model myopic behavior, while operating on the transition function can model optimism or the illusion of control. We parametrize each irrationality so that we can manipulate its ``intensity" or deviation from rationality. Figure \ref{fig:biases} summarizes our approach, which we detail below.

\subsubsection{Rational \demonstrator{}}
In our setup, the \textbf{rational} \demonstrator{} does value iteration using the Bellman update from \figref{fig:biases} with ground truth models and high $\gamma$. Our models change this update to produce different types of non-rational behavior.

\subsubsection{Modifying the max operator: Boltzmann}
 \noindent\textbf{Boltzmann}-rationality modifies the maximum over actions $\max_{a}$ with a Boltzmann operator with a parameter $\beta$:
\[V_{i+1}(s) = \textrm{Boltz}^\beta_{a}\sum_{s' \in S}  \transProb{(s')}\left(r_\theta(s,a,s') + \gamma  V_{i}(s') \right),\]
where $\textrm{Boltz}^\beta(\mathbf{x}) = {\sum_{i} x_i e^{\beta x_i}}/{\sum_i e^{\beta x_i}} $ \citep{ziebart2010modeling, asadi2017alternative}. 
This is the most popular stochastic model used in reward inference\citep{ziebart2010modeling, asadi2017alternative, fisac2017pragmatic}. After computing the value function, the Boltzmann-rational planner $d_{\textrm{Boltz}}$ returns a policy where the probability of an action is proportional to the exponential of the $Q$-value of the action: 
\[\pi(a|s) \propto e^{\beta Q_\theta(s,a)}.\]
The constant $\beta$ is called the \emph{rationality constant}, because as $\beta \to \infty$, the human choices approach perfect rationality (optimality), whereas $\beta=0$ produces uniformly random choices.

\subsubsection{Modifying the transition function}
Our next set of irrationalities manipulate the transition function away from reality.
\paragraph{Illusion of Control.}
People often overestimate their ability to control random events \citep{thompson1999illusions}. To model this, we consider \demonstrator{}s that use the Bellman update:
\[V_{i+1}(s) = \max_{a}\sum_{s' \in S}   \transProb^n{(s')} \left ( r_\theta(s,a,s') + \gamma  V_{i}(s')\right)\]
where $ \transProb^n{(s')} =\left (\transProb{(s')} \right)^n / \sum_{s'' \in S}\left (\transProb{(s'')} \right)^n$.  As $n\rightarrow \infty$, the \demonstrator{} acts as if it exists in a deterministic environment. As $n \rightarrow 0$, the \demonstrator{} acts as if it had an equal chance of transitioning to every possible successor state. 
\paragraph{Optimism/Pessimism}
Many people systematically overestimate or underestimate their chance experiencing of positive over negative events \citep{sharot2007neural}. We model this using \demonstrator{}s that modify the probability they get outcomes based on the value of those outcomes:
\[V_{i+1}(s) = \max_{a}\sum_{s' \in S}   \transProb^\omega{(s')}\left ( r_\theta(s,a,s') + \gamma  V_{i}(s')\right)\]
where $ \transProb^\omega{(s')} \propto  \transProb{(s')} e^{ \omega \left ( r_\theta(s,a,s') + \gamma  V_{i}(s')\right)}$. 
$\omega$ controls how pessimistic or optimistic the \demonstrator{} is. As $\omega \rightarrow +\infty$ (respectively, $\omega \rightarrow -\infty$), the \demonstrator{} becomes increasingly certain that good (bad) transitions will happen. As $\omega \rightarrow 0$, the \demonstrator{} approaches the rational \demonstrator{}.

\subsubsection{Modifying the reward}
Next, we consider \demonstrator{}s that use the modified Bellman update:
\[V_{i+1}(s) = \max_{a}\sum_{s' \in S}  \transProb{(s')}\left ( f(r_\theta(s,a,s')) + \gamma  V_{i}(s')\right)\]
where $f: \R \rightarrow \R$ is some scalar function. This is equivalent to solving the MDP with reward $f\circ r_\theta$, and allows us to model human behavior such as loss aversion and scope insensitivity. 

\paragraph{Prospect Bias} \citet{kahneman2013prospect} inspires us to consider a particular family of $f$s:
\[f_c(r) = \begin{cases}
\log(1+|r|) & r > 0\\
0 & r = 0 \\
-c \log(1+|r|) & r < 0
\end{cases}\]
$c$ controls how loss averse the \demonstrator{} is. As $c \rightarrow \infty$, the \demonstrator{} primarily focuses on avoiding negative rewards. As $c \rightarrow 0$, the \demonstrator{} focuses on maximizing positive rewards. 

\subsubsection{Modifying the sum between reward and future value:}
\paragraph{Extremal} People seem to exhibit duration neglect, sometimes only caring about the maximum intensity of an experience \citep{do2008evaluations}. We model this using Bellman update: 
 \begin{align*}
    \hspace*{-4pt} V_{i+1}(s) = \max_{a}\sum_{s' \in S}  \transProb{(s')} \max \begin{cases}
    r_\theta(s,a,s') \\
    (1-\alpha)  r_\theta(s,a,s')+ \alpha  V_{i}(s')
    \end{cases}
\end{align*}
These \demonstrator{}s maximize the expected maximum reward along a trajectory, instead of the expected sum of rewards. As $\alpha \rightarrow 1$, the \demonstrator{} maximizes the expected maximum reward received along the full trajectory. As $\alpha \rightarrow 0$, the \demonstrator{} becomes greedy and maximizes the immediate reward.

\subsubsection{Modifying the discounting}
\paragraph{Myopic Discount}
In practice, people are sometimes myopic, only considering near-term rewards. One way to model this is to decrease gamma in the Bellman update. At $\gamma = \gamma^*$, the discount rate specified by the environment, this is the rational \demonstrator{}. As $\gamma \rightarrow 0$, the \demonstrator{} becomes greedy and only acts to maximize immediate reward. 

\paragraph{Myopic Value Iteration}
As another way to model human myopia, we consider a \demonstrator{} that performs only $h$ steps of Bellman updates. That is, this \demonstrator{} cares equally about rewards for a horizon $h$, and discount to 0 reward after that. As $h \rightarrow \infty$, this \demonstrator{} becomes rational. If $h=1$, this \demonstrator{} only cares about the immediate reward.

\paragraph{Hyperbolic Discounting}
Human also exhibit hyperbolic discounting, with a high discount rate for the immediate future and a low discount rate for the far future \citep{grune2015models}. \citet{alexander2010hyperbolically} formulate this as the following Bellman update:
\[V_{i+1}(s) = \max_{a}\sum_{s' \in S} \transProb{(s')} \frac{r_\theta(s,a,s') + V_{i}(s') }{1 + k V_{i}(s')}\]
$k$ modulates how much the \demonstrator{} prefers rewards now versus the future. As $k \rightarrow 0$, this \demonstrator{} becomes a rational \demonstrator{} without discounting.

\section{Exploring the effects of known biases on reward inference}
\label{known-bias-effects}
Armed with our framework for characterizing irrationality, we test its implications for reward inference. We start by investigating the effects in random MDPs.

\subsection{Experimental Design: Exact Inference in Random MDPs.}
\label{random-mdp-setup}
\paragraph{Independent Variables.} We manipulate the type of the planner, and vary the degree parameters for each. We use different environments, sample different ground truth reward parameters, and test different trajectory lengths for the demonstrated behavior ($T=3$, $15$, and $30$ state-actions).

\paragraph{Dependent Measures.} 
To separate the inference difficulty caused by suboptimal inference from the difficulty caused by \demonstrator{} irrationality, we perform the exact Bayesian update on the trajectory $\xi$ \citep{ramachandran2007bayesian}, which gives us the posterior on $\theta$ given $\xi$,
$P(\theta|\xi) = \frac{P(\xi|\theta)P(\theta)}{\int_{\theta'} P(\xi|\theta')P(\theta')}$.
Our primary metric is the expected \textbf{log loss} of this posterior:
\begin{align*}
    \textrm{Log Loss}(\theta |\xi) = E_{{\theta^*}, \xi \sim \plannerParamed{\rewardParam^*}}\left[-\log P(\theta^*|\xi)\right].
\end{align*}
A low log loss implies that we are assigning a high likelihood to the true $\theta^*$. Note that in this case, the log loss is equal to the entropy of the posterior $H(\theta|\xi) = H(\theta) - \I(\theta; \xi)$. 

 For each environment and irrationality type, we calculate the performance of reward inference on trajectories of a fixed length $T$. To sample a trajectory of length $T$ from a \demonstrator{}, we fix $\theta^*$ and start state $s$. Then, we generate rollouts starting from state $s$ until $T$ state, action pairs have been sampled from $\pi = \plannerParamed{\rewardParam^*}$. We repeat this procedure 10 times for each start state. 

\begin{figure*}[t!]
    \centering
    \vspace{-0.5cm}
    \includegraphics[width=0.85\textwidth]{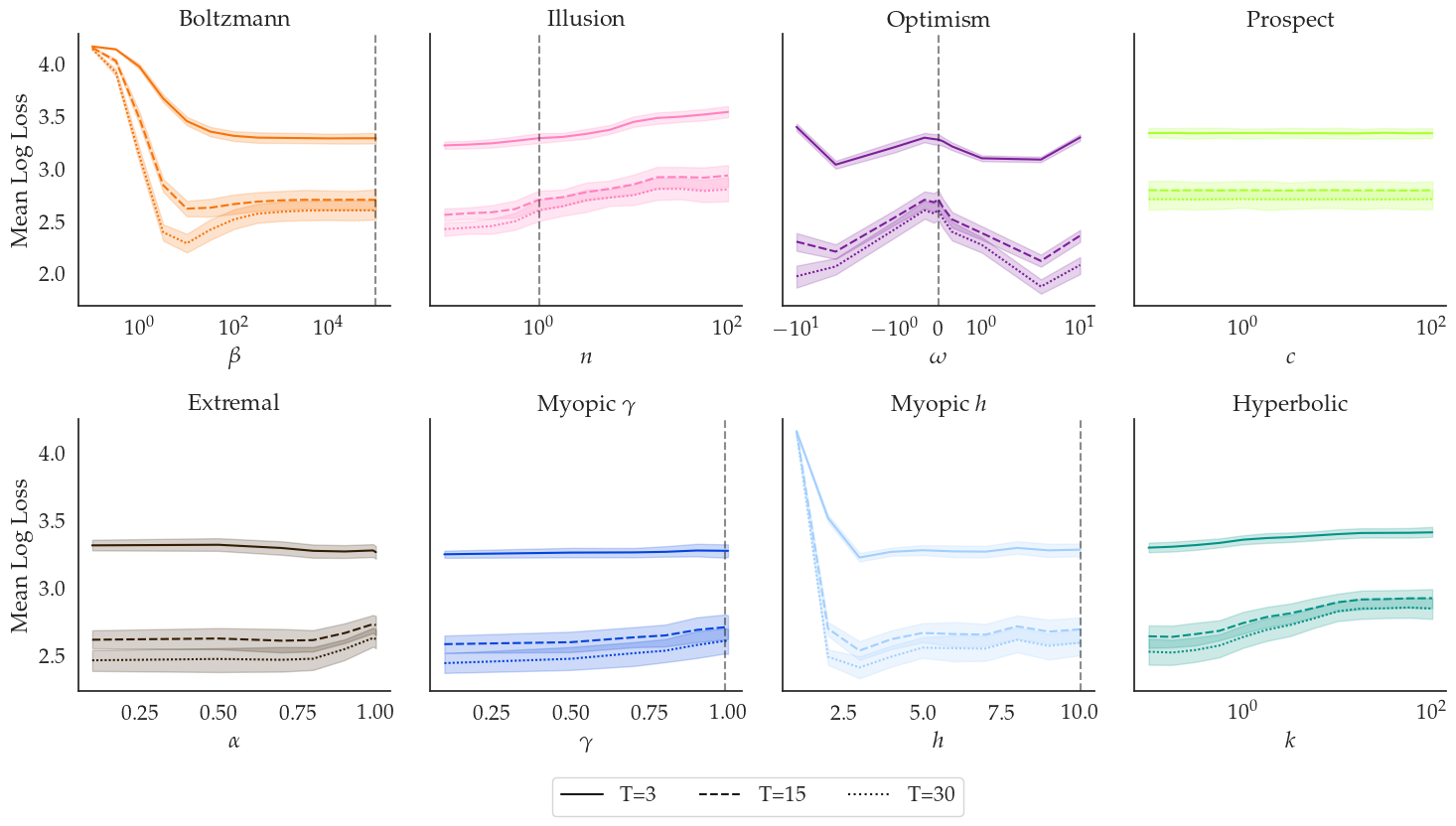}
    \caption{The log loss (lower = better) of the posterior as a function of the parameter we vary for each irrationality type, on the random MDP environments. For the irrationalities that interpolate to the rational planner, we denote the value that is closest to rational using a dashed vertical line. Every irrationality except Prospect Bias all have parameter settings that outperform the rational planner. The error bars show the standard error of the mean, calculated by 1000 bootstraps across environments.}
    \label{fig:random_mdp_params_vs_log_loss}
\end{figure*}

\paragraph{Simulation Environment.}
We used MDPs with 10 states and 2 actions, where each $(s,a)$-pair has 2 random, nonzero transition-probabilities. We used $\gamma = 0.99$ and start trajectories in every state without reward. In these $\theta$ is a vector of length 3, where $\theta_i$ is the reward received from transitions out of state $i$ (and all other rewards are 0). We discretized each $\theta_i$ with 4 values, leading to $|\Theta| = 64$. We generated 20 such random MDPs. 

\begin{figure*}[t!]
    \centering
    \includegraphics[width=0.85\textwidth]{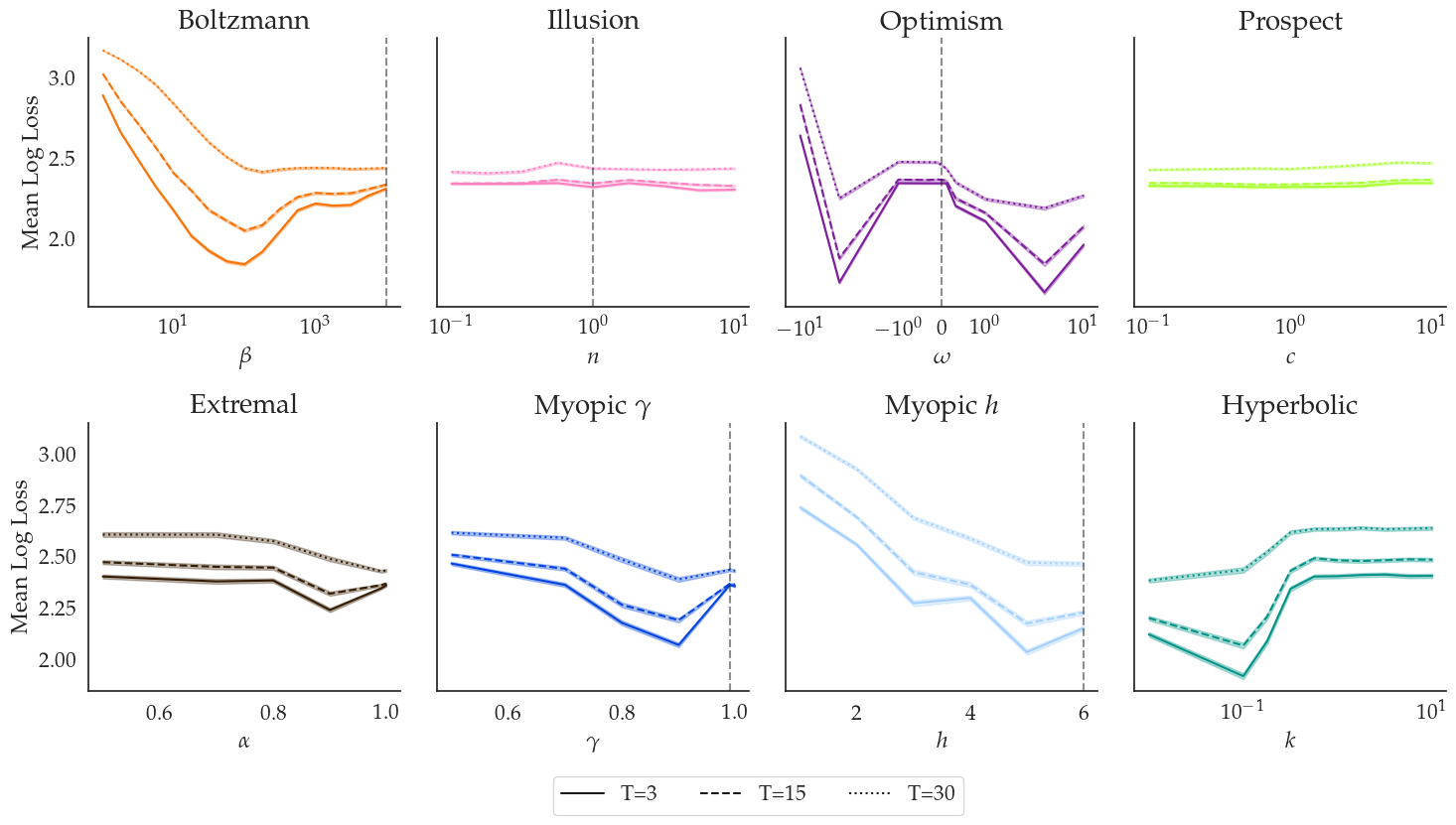}
    \caption{The analog of \figref{fig:random_mdp_params_vs_log_loss} for the gridworld. Error bars are the standard error of the mean. The findings are surprisingly similar as with the random MDPs. Note the more limited x-axis ranges we used in this experiment. } 
    \label{fig:gridworld_params_vs_log_loss}
\end{figure*}

\subsection{Results}
\label{exact-inference-results}
\figref{fig:random_mdp_params_vs_log_loss} plots the log loss for each irrationality for random MDPs. The degree affects reward inference, with many settings naturally resulting in worse inference, especially at the extremes. However, what is surprising is that \emph{every} type except Prospect Bias has at least one degree (parameter) setting that results in \emph{better} inference with enough data: we see that most irrationality types can be more informative than rational behavior. The more data we have, the more drastic the difference (T=30 results in both better inference and larger difference relative to rational). 

We put this to the test with a repeated-measures ANOVA with planner type as the independent variable, using the data from the best parameter setting from each type and T=30, and environment as a random effect. We find a significant main effect for planner type ($F(8, 806372) = 6102.93,\ p < 0.001$). A post-hoc analysis with Tukey HSD corrections confirmed that every irrationality except Prospect improved inference compared to the fully rational planner ($p<.001$). For $T=30$, Optimism with $\omega = 3.16$ performed best, followed by Illusion of Control with $n=0.1$ and Boltzmann with $\beta = 10$.

\begin{figure}
    \centering
    \includegraphics[width=0.8\columnwidth]{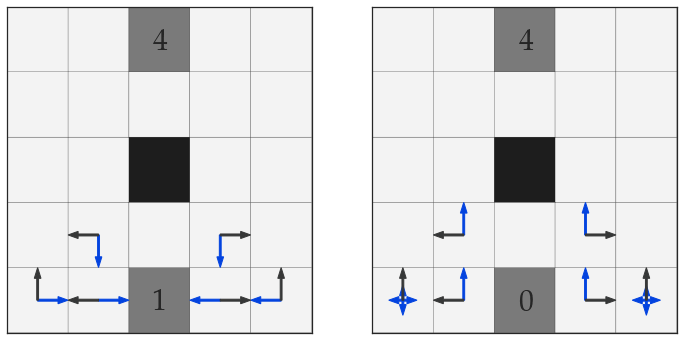}
    \caption{Myopic value iteration ($h=5$), here presented in blue) produces different policies for $\theta^*=(4,1)$ vs. $\theta^*=(4,0)$: while the rational expert (whose policy is represented in black) always detours around the hole and attempts to reach the larger reward, myopia causes the myopic expert to go for the smaller source of reward when it is non-zero.}
    \label{fig:myopic_vs_optimal}
\end{figure}
\begin{table*}[t!]
  \caption{Autonomous Driving Results: Merging}
  \label{tab:car-results}
  \centering
  \begin{tabular}{c cc cc }
    \toprule
    \multicolumn{1}{c}{~} & \multicolumn{2}{c}{Bayesian IRL} & \multicolumn{2}{c}{CIOC} \\
    \cmidrule(r){2-3}
    \cmidrule(r){4-5}
    Horizon & Log Loss & Information Gain & Cosine Similarity & Normalized Return \\
    \hline 
    3 & \textbf{0.690} & \textbf{0.696} & \textbf{0.999} & 0.939\\
    5 & 0.824 & 0.562 & 0.940 & \textbf{0.981}\\
    7 & 1.383 & 0.004 & 0.350 & 0.856\\
    \bottomrule
  \end{tabular}
\end{table*}

\section{Does this effect generalize?}
\label{generalization}
In the domain of random MDPs, we found that not only does irrationality affect reward inference, certain irrationalities can actually result in better inference. In this section, we probe the generality of this finding empirically, and explain it theoretically.

\subsection{Gridworld}
\label{gridworld}
Random MDPs lack structure, so we first test a toy environment that adds natural structure based on OpenAI Gym's `Frozen-Lake-v0' \citep{brockman2016openai} (\figref{fig:myopic_vs_optimal}). \figref{fig:gridworld_params_vs_log_loss} shows the results in the gridworld: they are surprisingly (and reassuringly) similar to the random MDPs. This suggests that irrationalities can indeed help inference across differently structured MDPs. By inspecting the policies, we see that the irrational \demonstrator{}s were able to outperform the rational \demonstrator{} by disambiguating between $\theta$s that the rational \demonstrator{} could not. To visualize this, we give an example of how the policy of several irrational \demonstrator{}s differ on the gridworld when the rational \demonstrator{}'s policies are identical in  \figref{fig:myopic_vs_optimal}. Additional examples are given in \figref{fig:optimism_vs_opt} and \figref{fig:boltzmann_vs_optimal} in the appendix.


\subsection{Autonomous Driving}
\label{cars}
Even if known irrationality helps reward inference both empirically and theoretically in small MDPs, the question still remains whether this result will matter in practice. Next, we investigate the effect of \demonstrator{} irrationality on reward inference in a continuous state and action self-driving car domain \citep{sadigh2016planning}. Switching to this real-world domain means we can no longer plan exactly, so we use model predictive control. It also means we can no longer run exact inference, so we approximate Bayesian IRL through samples, and also test continuous IOC (CIOC, \citet{levine2012continuous}) which recovers an approximate MLE. We measure the cosine similarity between the estimate and $\theta^*$, as well as the normalized reward of the trajectory when optimizing the recovered estimate. We use the merging task from \figref{fig:front} (more details in the appendix).

We report our results in Table \ref{tab:car-results}. As the columns relating to Bayesian IRL show, decreasing the MPC planning horizon significantly decreases the log loss of the Bayesian IRL posterior. As with the gridworld results in \secref{gridworld}, the reason for this is that the shorter planning horizons exhibit more diverse behavior (as a function of the reward). Of the 4 reward settings we used, MPC with horizon 3 and 5 produced two different qualitative behaviors, whereas all rewards led to qualitatively similar behavior with horizon 7 (depicted in \figref{fig:front}). When the weight on target speed was large enough, MPC with horizon 3 and 5 would produce a trajectory that overtakes the front car by going off the road. In all other cases, the \demonstrator{} would merge between the two other cars. 

\begin{figure*}
    \begin{subfigure}{0.65\textwidth}
    \centering
    \includegraphics[width=1.0\textwidth]{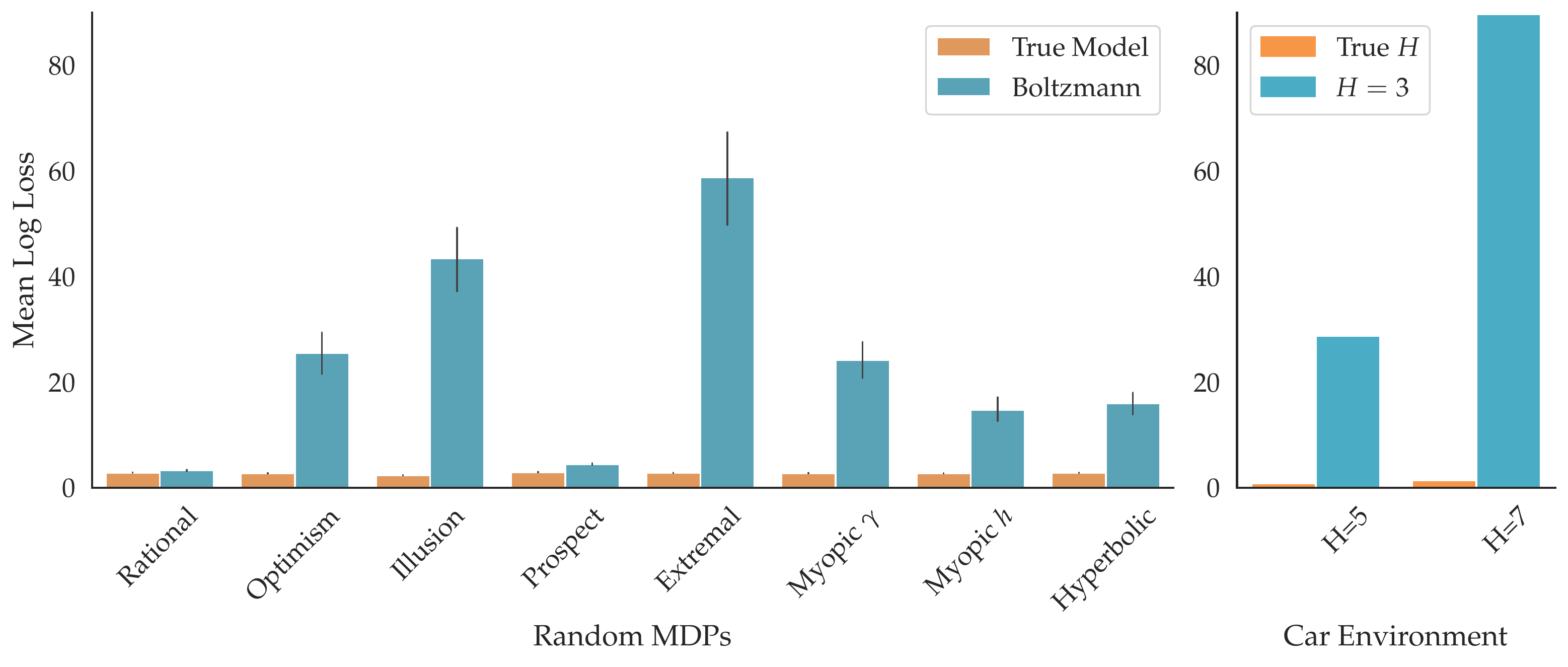}
    \caption{}
    \label{fig:misspec_bad}
    \end{subfigure}
    \hfill
    \begin{subfigure}[]{0.35\textwidth}
    \centering
    \includegraphics[width=0.8\columnwidth]{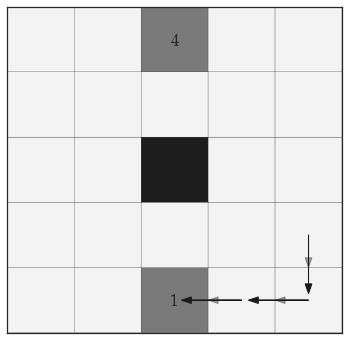}
    \caption{}
    \label{fig:bad_inference}
    \end{subfigure}
    \caption{(a) A comparison of reward inference using a correct model of the irrationality type, versus always using a Boltzmann-rational model ($\beta = 10$), on the random MDPs (left) and the car environment (right).  The impairment due to model misspecification greatly outweighs the variation in inference performance caused by various irrationalities. The error bars show the standard error of the mean, calculated by the bootstrap across environments. (b)
    An example of why assuming Boltzmann is bad when the ground truth \demonstrator{} is Myopic in the gridworld environment - the Boltzmann rational agent would take the trajectory depicted only if the reward at the bottom was not much less than the reward at the top. A myopic \demonstrator{} with $n \leq 4$, however, only "sees" the reward at the bottom. Consequently, inferring the preferences of the myopic agent as if it were Boltzmann leads to poor performance in this case. }
\end{figure*}
\begin{figure*}[]
    
\end{figure*}

We found similar results for CIOC: decreasing the MPC planning horizon increases the cosine similarity between the true and CIOC-recovered rewards. The return of optimizing the CIOC-recovered rewards is also higher when the \demonstrator{}'s planning horizon is 3 or 5 than 7. 

Overall, yet again we find that irrationality (in this case specifically myopia) improves reward inference by producing more diverse behavior as a function of the reward.

\subsection{Theoretical Analysis}
\label{theory}
We now investigate this phenomenon theoretically. We show that not only can rational behavior be arbitrarily less informative than irrational behavior, but also that this applies to Boltzmann-rationality.

\subsubsection{Informativeness as mutual information.}
The mutual information $\I(\rewardParam; \planner(\rewardParam)) = H(\rewardParam) - H(\rewardParam|\plannerParamed{\theta})$
between the policy and the reward parameters allows us to quantify the informativeness of a planner. As the conditional entropy $H(\rewardParam|\planner(\rewardParam)) = E\left [-\log p(\theta | d(\theta) \right ]$ is equal to the log loss of the posterior under the true model, the mutual information upper bounds how much better (in terms of log loss) \emph{any} inference procedure can do, relative to the prior. 
A known insight is that planners that optimize for maximizing the discounted sum of rewards \emph{are not the same} as those that optimize for being informative \citep{dragan2013legibility, hadfield2016cooperative}:
$\arg\max_\planner E_\rewardParam [V^{\plannerParamed{\rewardParam}}_\rewardParam] \not= \arg\max_{\planner} \I(\theta; \plannerParamed{\rewardParam})$.
{While this insight has been tested empirically, we begin with a theoretical lens for understanding it.}

\subsubsection{Irrationalities exist that are arbitrarily better for inference than rationality.}
We first consider \textbf{deterministic planners}: planners that return deterministic policies. 
We show that there are cases where the rational behavior is not informative at all, whereas some (irrational) deterministic planner achieves the theoretical upper bound on informativeness. 
\begin{prop}
There exists a family of URMDPs with state spaces of any size, such there exists a deterministic planner $\planner_{\textrm{Irrational}}$ satisfying $\I(\theta, d_{\textrm{Irrational}}(\theta)) = \log|\Theta|$ and $\I(\theta, d_{\textrm{Rational}}(\theta)) = 0$. 
\end{prop}


\subsubsection{Boltzmann-rationality is (arbitrarily) more informative than full rationality.} Of course, the upper bound is attained by \emph{some} irrational planner. It demonstrates that an \demonstrator{} can perform better than rational when specifically optimizing for informativeness. But this is an artificial, contrived kind of irrationality. In fact, prior work that maximized informativeness did so by solving a \emph{more difficult} problem than rationality \citep{dragan2013legibility, hadfield2016cooperative}. Here, we provide evidence that Boltzmann-rationality, a standard model of stochastic choice, outperforms full rationality for reward inference. (Proof in appendix.)

\begin{prop}
There exists a family of one-state two-action MDPs, with arbitrarily large $|\Theta|$ such that $\I(\theta, d_{\textrm{Boltz}}(\theta)) = \log|\Theta|$ and $\I(\theta, d_{\textrm{Rational}}(\theta)) = 0$. 
\end{prop}

\section{Effects of misspecification on reward inference}
\label{misspec-effects}
We see that irrationalities sometimes hinder, but sometimes help reward inference. So far, the \inferer{} had access to the type (and degree of irrationality) during inference. Next, we ask how important it is to know this. Can we not bother with irrationality, make the default assumption (of Boltzmann-rationality), and run inference?

\subsection{Assuming noisy rationality can lead to very poor inference}\figref{fig:misspec_bad} suggests that the answer is no. We start by comparing inference with the true model on random MDPs versus with assuming the standard Boltzmann model as a default. The results are quite striking: not knowing the correct irrationality harms inference tremendously.\footnote{\citet{shah2019feasibility} proposed a way to model irrationality and analyzed its benefit over assuming Boltzmann; the benefit was very limited, which they attributed due to their deep learning model's brittleness compared to exact planning. Here, we compare to perfect modeling to analyze the headroom that modeling has.} We then confirm that misspecification greatly harms inference in the autonomous driving environment. In \figref{fig:bad_inference}, we provide an example of why assuming Boltzmann-rationality when the \demonstrator{} is Myopic leads to bad inference. 
This emphasizes the importance of understanding irrationality when doing reward inference going forward.

\subsection{Approximate models of irrationality might be enough} This finding is rather daunting, as perfect models of irrationality are very challenging to develop. But do they need to be perfect? Our final analysis suggests that the answer is no as well. In \figref{fig:additional_misspec_grid}, we report the log loss of inference with the correct type, but under misspecification of the parameter. Encouragingly, we find that in many cases, merely getting the type of the \demonstrator{} irrationality correct is actually sufficient to lead to much better inference than assuming Boltzmann rationality. Further, we also find evidence that the \inferer{} does not need to get the type exactly right either: as shown in \figref{fig:myopia_type_misspec}, if the \inferer{} accounts for the \demonstrator{}'s myopia, but gets the \emph{type} of the myopia wrong, this still leads to significantly better inference than assuming Boltzmann rationality. 
\label{additional-misspec}

\begin{figure*}[]
    \centering
    \includegraphics[width=0.8\textwidth]{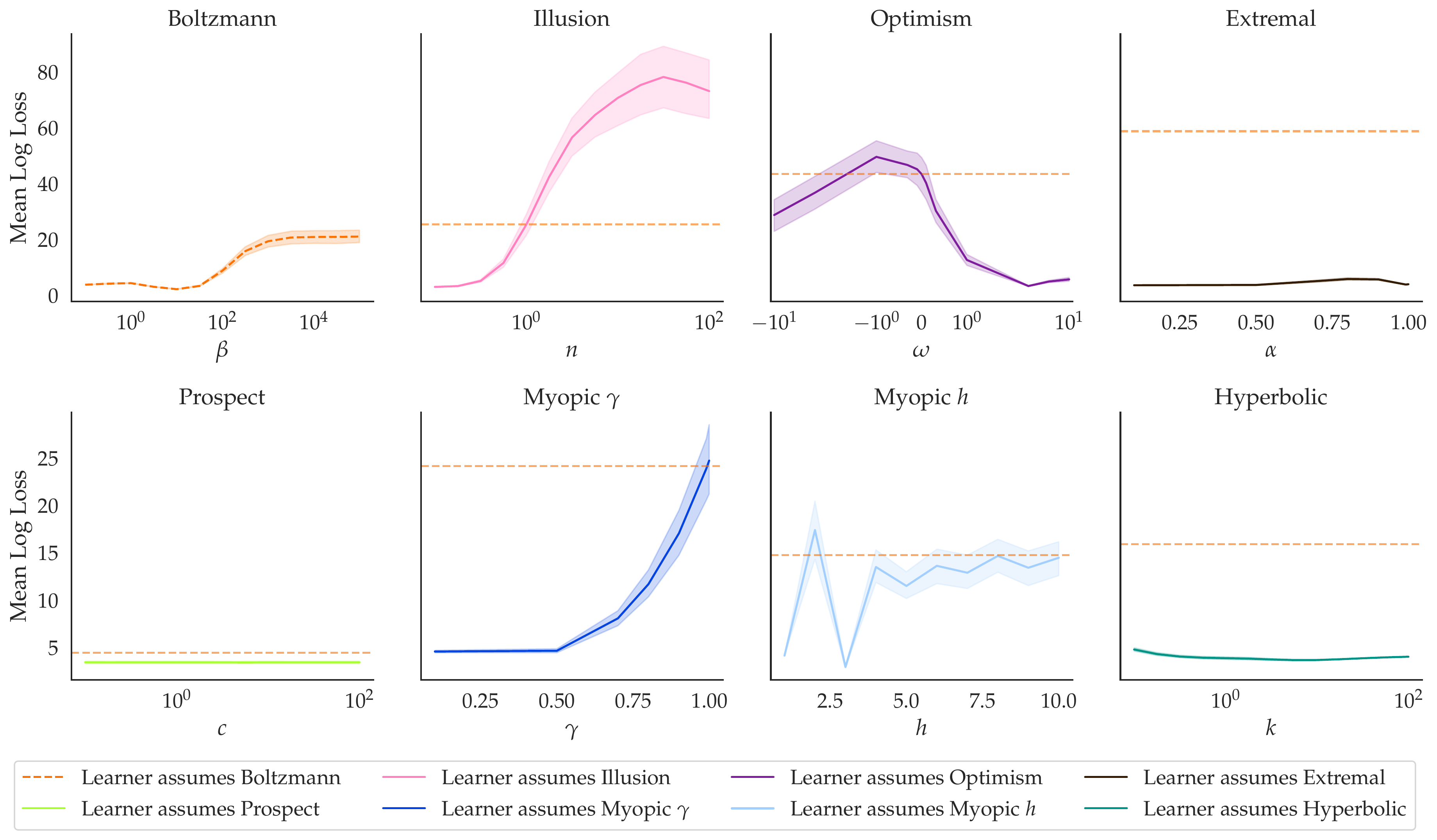}
    \caption{The log loss (lower = better) of various models under parameter misspecification. Each x-axis shows the parameter that the \inferer{} assumes. The orange line represents the performance when the \inferer{} makes the faulty assumption that the \demonstrator{} is Boltzmann-rational. In many cases, the \inferer{} perform better than by assuming Boltzmann-rational just by getting the type of the planner correct, even if they don't get the exact parameter correct. The error bars show the standard error of the mean, calculated by the bootstrap across environments.} 
    \label{fig:additional_misspec_grid}
\end{figure*}

\begin{figure}[]
    \centering
    \includegraphics[width=1.\columnwidth]{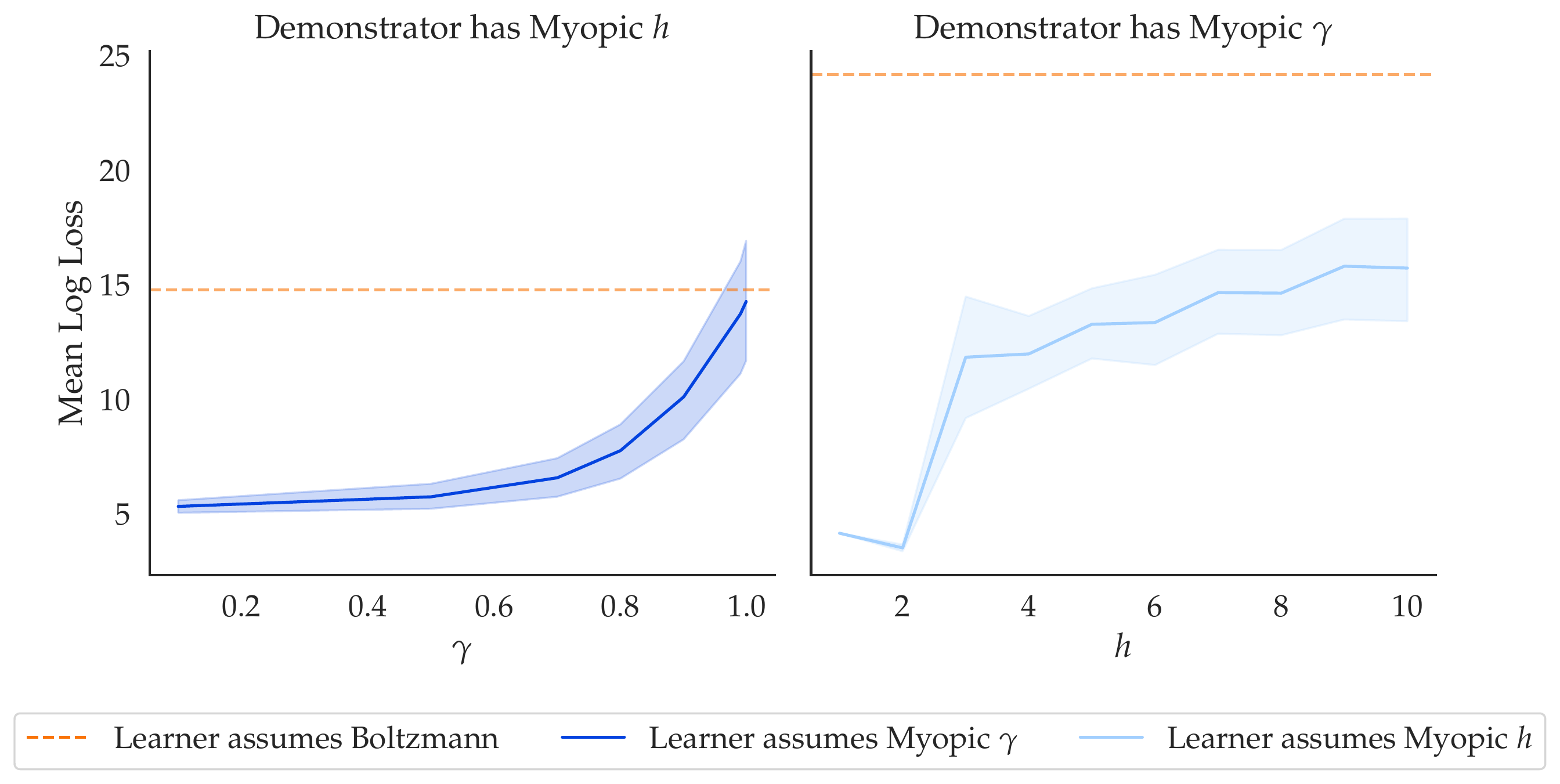}
    \caption{ The log loss (lower = better) of two myopic \demonstrator{}s under type misspecification. On the left, the \demonstrator{} performs myopic value iteration (Myopic $h$), but the \inferer{} assumes the \demonstrator{} has a myopic discount rate $\gamma$ (Myopic $\gamma$). On the right, the \demonstrator{} has a myopic discount rate $\gamma$ but the \inferer{} assumes myopic value iteration. However, in both cases, this leads to better inference than assuming Boltzmann-rationality.} 
    \label{fig:myopia_type_misspec}
\end{figure}

\section{Discussion}

\subsubsection{Summary}
In this work, we analyzed how the \emph{"ground truth"} human pairs with the robot's \emph{model} of the human to affect the performance of reward inference. To investigate this, we defined irrationalities as deviations from the Bellman update (\secref{irrationality-types}). 

On the \emph{model} side, we found that if the human is irrational, incorrectly modeling them via the standard noisy-rationality assumption can lead to so poor performance of reward inference, that it could be better to stick with the prior and skip inference. While it is natural for incorrect models to lead to worse performance, this gap is particularly daunting. Fortunately, 
 even approximate models of irrationalities may be enough - merely getting the \emph{type} of the \demonstrator{} irrationality correct led to significantly better inference than assuming noisy-rationality. 

On the \emph{ground truth} human side, we asked what the best-case performance of learning is for different human irrationalities -- that is, if the \inferer{} models the human correctly, how well can it hope to infer the reward? We expected that noisy and irrational humans would be less informative about what they want than perfectly rational humans, and therefore make the reward inference problem inherently more difficult. However, we found that irrationality can actually make human behavior \emph{more} informative, by increasing the mutual information between reward and policy. In three domains, we found that irrational behavior varies more as a function of the reward, and ultimately results in better learning performance. This is not only a theoretically interesting observation about the upper bound of the problem, but also encourages further research on modeling human irrationality, so that \inferer{}s can benefit from this additional information in human behavior.


\subsubsection{Limitations and Future Work}
Much of our results are in the relatively random MDP and gridworld environments. Though we confirmed our key results in a higher-dimensional, continuous autonomous driving task, future work is needed to replicate these results in more realistic domains. Future work is also needed to develop reward learning algorithms that can account for a variety human biases in larger domains.  

Further, our paper presents an analysis, not a practical algorithm for improving reward learning. Our most interesting finding, and the one we dedicate most of the paper to, is the fact that irrationalities can improve reward inference. Though we showed that even approximate models of irrationality suffice to outperform the standard noisily-rational assumption, actually benefiting from irrationality still requires precise models of the \demonstrator{}'s irrationality.  Nonetheless, we think the result is still important, because it encourages treating biases more seriously. In the end, our paper makes a perhaps fortunate observation -- that the irrational human, when correctly modelled, is better than the rational human correctly: since real humans are likely not perfectly rational, this might be good news in the end. In fact, is is probably easier to learn irrationality models of people, than to make people act rationally -- so in the end, it is a good thing that irrational behavior is the more informative one.

What is more, our results also point to an untapped opportunity: training people to act irrationality (e.g. myopically) when demonstrating behavior in order to be more informative to our robots. User studies with real people will be important to understand how to make use of this opportunity.

\newpage

\bibliography{references.bib}
\bibliographystyle{plainnat}

\newpage

\appendixpage
\section{Additional Details for Random MDPs}
\label{full-randomMDPs}
To calculate the policies corresponding to particular planners with a particular $\theta$, we perform value iteration using the Bellman equation until convergence (defined as $||V_{i+1} - V_{i}||_\infty < 0.001$). 

\section{Details for Gridworld}
\label{full-gridworld}
In this section, we investigate the effects of irrationality on inference in an MDP based on OpenAI Gym's `Frozen-Lake-v0' \citep{brockman2016openai}. This a small 5x5 gridworld (\figref{fig:myopic_vs_optimal}), consisting of three types of cells: ice, holes, and rewards. The \demonstrator{} can start in any ice cell. At each ice cell, the \demonstrator{} can move in one of the four cardinal directions. With probability $0.8$, they will go in that direction. With probability $0.2$, they will instead go in one of the two adjacent directions. Holes and rewards are terminal states, and return the \demonstrator{} back to their start state. They receive a penalty of $-10$ for falling into a hole and $\theta_i \in [0,4]$ ($|\Theta| = 25$) for entering into the $i$th reward cell. As with the random MDPs, to calculate the policies corresponding to particular planners, we perform value iteration using the Bellman equation until convergence (defined as $||V_{t+1} - V_{t}||_\infty < 0.001$)

To visualize why inference quality is improved,  we show examples of how the policy of several irrational \demonstrator{}s differ on the gridworld when the rational \demonstrator{}'s policies are identical in \figref{fig:optimism_vs_opt}, \figref{fig:myopic_vs_optimal}, and  \figref{fig:boltzmann_vs_optimal}. 

Finally, \figref{fig:bad_inference} example of why using the wrong model for reward inference leads to bad inference. In it, the reward inference algorithm assumes that the \demonstrator{} in Boltzmann when it is actually Myopic. The Boltzmann rational agent would take this trajectory only if the reward at the bottom was not much less than the reward at the top. The myopic agent with $n \leq 4$, however, only "sees" the reward at the bottom. Consequently, inferring the preferences of the myopic agent as if it were Boltzmann leads to poor performance in this case.

\begin{figure}[t]
    \centering
    \includegraphics[width=0.8\columnwidth]{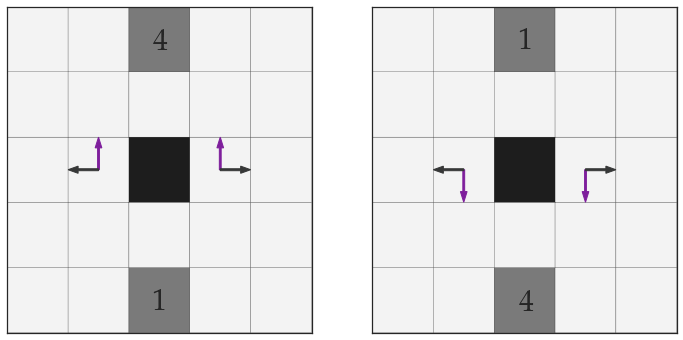}
   
    \caption{Optimism bias ($\omega = 3.16$) produces different actions for $\theta^*=(4,1)$ vs. $\theta^*=(1,4)$ in the states shown: the rational policy is to go away from the hole regardless of $\theta$, but an optimistic \demonstrator{} takes the chance and goes for the larger reward -- up in the first case, down in the second.} \label{fig:optimism_vs_opt}
    \end{figure}

\begin{figure}
    \centering
    \includegraphics[width=0.8\columnwidth]{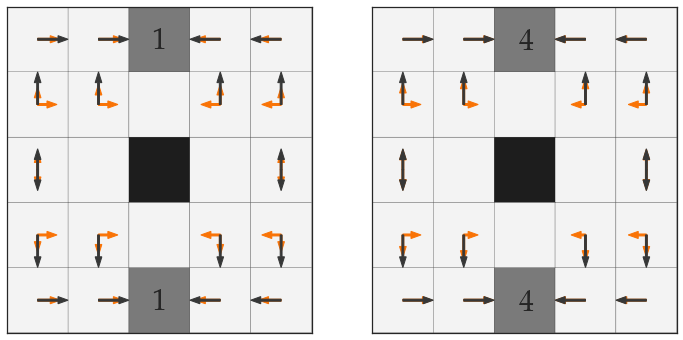}
    \caption{Boltzmann ($\beta=100$) produces different policies for $\theta^*=(1,1)$ vs. $\theta^*=(4,4)$: when $||\theta||$ is larger, the policy becomes closer to that of the rational \demonstrator{}, as the differences in $Q$-values becomes larger. }
    \label{fig:boltzmann_vs_optimal}
\end{figure}

\section{Details for Autonomous Driving Environment}
\label{full-cars}

Even if known irrationality helps reward inference both empirically and theoretically in small MDPs, the question still remains whether this result will matter in practice. As a result, we investigate the effect of \demonstrator{} irrationality on reward inference in the self-driving car domain \citep{sadigh2016planning}.

\noindent \textbf{Simulation environment.} As in previous work in the car domain \citep{sadigh2016planning, dorsa2017active}, we model the dynamics of cars using a point-mass model. The state of each car is a 4-dimensional vector $\state = [ x ~ y ~h ~v]$, where $x, y$ are the coordinates of the car, $h$ is the heading, and $v$ is the speed. The control input for the car is a two dimensional vector $\controls = [u_1 ~ u_2]$, where $u_1$ is the steering input and $u_2$ is the acceleration. We also include a friction coefficient $\alpha$. The dynamics model of the vehicle is:
\[[\dot{x} ~ \dot{y} ~\dot{h} ~\dot{v}] = [v\cdot\cos{h} ~~~~ v \cdot\sin{h} ~~~~ v \cdot u_1 ~~~~u_2 - \alpha\cdot v]. \] 

For ease of simulation, we discretized the simulation along the time dimension using the following dynamics:
\begin{align*}
    \Delta{\state}_t &= [\Delta{x}_{t} ~ \Delta{y}_{t} ~\Delta{h}_{t} ~\Delta{v}_{t}] \\
    &= [\bar v_t \cos{h_t} dt ~~ \bar v_t \sin{h_t} dt ~~ \bar v_t  u_1  dt ~~ (u_2 - \alpha v_t)  dt],
\end{align*}
where $ \Delta{\state}_t = \state_{t+1} - \state_t $ and $\bar v_t = v_t + 0.5 \cdot u_2 \cdot dt$. In our experiments, we used $dt = 0.1s$.

Reward functions in this environment are assumed to be a linear combination of features:
\[r_{\theta}(s,a) = \theta^\top f(s,a), ~~f:\states \times \actions \rightarrow \R^n\]

Due to environment complexity, we can't solve for optimal trajectories in our environment directly. Instead, we suppose that the planner of is performing Model Predictive Control (MPC) at every iteration -- that is, it will plan a finite horizon sequence of actions to maximize its reward, execute the first action in the sequence, then replan. As a basic model of irrationality, we consider shorter planning horizons. We assume that the reward inference procedure knows the planninng horizon exactly. 

\paragraph{Bayesian IRL in the car environment} First, we consider the analogue of the results in \secref{exact-inference-results} in the car environment: what is the log-loss of the posterior on $\rewardParam$, given different planning horizons? Since the demonstrations used in this work are generated via trajectory optimization (and not an inherently stochastic process), using Bayesian IRL requires us to specify a ``fake" observation model $P(\xi|\theta)$. 

We use the following distribution, normal in feature space, for $P(\xi| \theta)$:
\[P(\xi|\theta) \propto e^{\left (E[f(\xi)] - E[f(\xi^*_{\theta}] \right )^T \Sigma_{\xi_{\theta}^*}^{-1} \left (E[f(\xi)] - E[f(\xi^*_{\theta}]\right )},\]
where $E[f(\xi^*_{\theta}]$ is the expectation of the features of the optimal trajectory $\xi^*_\theta$ for $r_\theta$.

\paragraph{Maximum likelihood IRL in cars}
In practice, full Bayesian inference is completely intractable for complicated domains such as cars. Instead, the state of the art for reward inference are approximate, maximum-likelihood estimate (MLE) based methods.To study the effects of irrationality on MLE-based inference methods, we also perform Levine and Koltun's Continuous IOC with Locally Optimal Examples (CIOC) \citep{levine2012continuous}, which uses a Boltzmann model of the \demonstrator{} with a second order Laplace approximation of the normalizing constant. Since our \demonstrator{}s are actually locally optimal, we rectify the bias induced by the Boltzmann model by using a large $\beta = 10^4$. To rectify the issue of local optimization in CIOC, we initialize the optimization procedure with the true reward weights. 

As CIOC returns an (approximate) MLE estimate for the reward parameter $\theta$, we cannot use the log-loss metric. Instead, we evaluate the reward functions by the \textbf{cosine similarity} of $\theta^*$ and $\hat \theta$, as well as the \textbf{normalized return} of the trajectory when optimizing the recovered $\hat \theta$. (We normalize the returns so that the optimal trajectory has return $1$ and the trajectory that goes forward at constant speed $0$.)

\noindent{\textbf{Driving Scenario: Merging}} Our experiments were performed in a simple merging environment (Fig. \ref{fig:front}). In it, the \demonstrator{} wants to merge into the right lane while trying to maintain its $1.2$ forward speed. In addition to the \demonstrator{} car, the right lane contains two constant velocity cars, traveling at $0.8$ speed. The features of this environment are composed of a squared penalty for deviating from $1.2$ forward speed, features for the squared distances to the medians of each of the lanes, a feature for the minimum squared distance to any of the medians of the lanes, and a smooth collision feature. 

For the Bayesian IRL scenario, we considered four different reward functions on this domain, consisting of varying the weight on the penalty for deviating from the target speed. All the other weights are unchanged. In particular, we considered $\theta_{\textrm{speed}} \in \{0.5, 1, 2, 4\}$. In the CIOC scenario, we used $\theta_{\textrm{speed}} = 1$. We also consider 3 different planning horizons: $h=3$, $h=5$, and $h=7$. This means we had 12 conditions for Bayesian IRL and 3 for CIOC. 

\section{Proofs for \secref{theory}}
First, note that for deterministic planners, there that there are $|\actions|^{|\states|}$ such policies, and thus we have $\I(\theta; d(\theta)) \leq H(d(\theta)) \leq |\states| \log |\actions|$ (as the entropy of a discrete random variable $X$ is bounded above by the logarithm of the size of its support $|X|$). Similarly, we also have that $\I(\theta; d(\theta)) \leq H(\rewardParam) \leq \log|\rewardParams|$. For deterministic planners, we thus have
\begin{align*}
    \I(\theta; d(\theta)) \leq \min(\log|\rewardParams|, |\states| \log |\actions|),
\end{align*}

We will now prove a stronger version of proposition 3. 
\begin{prop}
There exists a family of URMDPs with state and action spaces of any size, such that the rational planner provides no information, and there exists a deterministic planner that provides $min(\log |\rewardParams|, |\states| \log |\actions|)$ bits worth of information. 
\end{prop}

\begin{proof}
Consider a set of environments where $\rewardParams = \{1, 2, ..., |\actions|^{|\states|}\}$, the prior $\prior$ is uniform over $\rewardParams$ (and thus $H(\theta) = \log |\rewardParams|$), and where
\begin{align*}
    r_\rewardParam(s, a, s') &= \begin{cases} \theta & a = a^* \\
    0 & a \not = a^*
    \end{cases}
\end{align*}
for some $a^*$. Then the unique optimal policy has $\pi(s) = a^*$ for every $\rewardParam$ (as $\theta > 0$). This implies that
\[\I(\theta, d_{\textrm{Rational}}(\theta)) \leq  H(d_{\textrm{Rational}}(\theta)) = 0.\]

However, the planner $d':\theta \mapsto \pi^{(\theta)}$, with some fixed some ordering $\{\pi^{(i)}\}_{i \in \{1, 2, ..., |\actions|^{|\states|}\}}$ of the possible policies, satisfies $H(\theta| d'(\theta)) = 0$ and thus $\I(\theta; d'(\theta)) = H(\theta) = \log |\rewardParams| = |\states| \log |\actions|$.
\end{proof}

By choosing $A$ such that $|A| \geq |\theta|^{1/|S|}$, we get the Proposition 3 in the text.  

\begin{prop}
There exists a family of one-state two-action MDPs, with arbitrarily large $|\Theta|$ such that $\I(\theta, d_{\textrm{Boltz}}(\theta)) = \log|\Theta|$ and $\I(\theta, d_{\textrm{Rational}}(\theta)) = 0$. 
\end{prop}
\begin{proof}
Suppose first action $a_1$ has reward $\theta$, where $\rewardParam \in \{1, 2, ..., -|\rewardParams|\}$, and the second action $a_2$ has reward $0$. Then the rational planner $d_{\textrm{Rational}}$ will always return a policy that always takes the first action, while $ d_{\textrm{Boltz}}(\theta)$ is the policy
\begin{align*}
    \pi_\theta(a_1|s_1) &= \frac{e^{\beta \left(\theta + V(s_1) \right)}}{e^{\beta\left(\theta + V(s_1)\right)} + e^{\beta V(s_1)}} = \frac{1}{1+e^{-\beta\theta}}\\
    \pi_\theta(a_2|s_1) &= \frac{1}{1+e^{\beta\theta}}
\end{align*}
This mapping is injective, and so $H(\theta| d_{\textrm{Boltz}}(\theta)) = 0$ and thus $\I(\theta; d_{\textrm{Boltz}}(\theta)) = H(\theta) = \log |\rewardParams|$.
\end{proof}

\section{More theory}
\paragraph{Why recover the reward parameter?}
Even if irrationalities can help with inference, a natural question why we wish to infer the reward. If the actual goal is acting optimally, a rational \demonstrator{} would be great because we can just imitate them. In addition to the fact that real people are not perfectly rational, the reason we focus on inference is that imitation is not enough in some cases, e.g. if we need to behave in new environments \citep{taylor2009transfer,devin2017learning, jing2019task}, or transfer the policy to a robot with different dynamics \citep{cully2015robots, fu2018learning, reddy2018you}.

In fact, there are cases where being robust to changes in dynamics requires identifying the reward parameters. Indeed, there exist such cases where it is impossible to identify the parameters from rational \demonstrator{}s, but possible from Boltzmann \demonstrator{}s:

\begin{prop}
There exists an URMDP $M$ and a set of new transition probabilities $\{\{\transProb^{(i)}\}\}_{i \in \{1, 2, ..., |\rewardParams|\}}$ such that $\log|\rewardParams|$ bits of information are needed to compute the optimal policy under all the transition probabilities; and where $\I(\theta, d_{\textrm{Rational}}(\theta)) = 0$, whereas $\I(\theta, d_{\textrm{Boltz}}(\theta)) = \log|\rewardParams|$.

\end{prop}
\begin{proof} We construct a 2 state, 2 action URMDP.
Label the two states $\state_{1}$ and $\state_{2}$. Let $\rewardParams = \{1, 2, ..., |\rewardParams|\}$, the prior be uniform over $\rewardParams$, and let $r(s, a, s_1) = \theta$ for all $s, a$ and $r(s, a_1, s_2) = 0$, $r(s, a_2, s_2)=-1$ for all $s$.\\ 
Under the original transition probabilities  $\{\transProb\}$, $a_1$ leads deterministically to $s_1$ while  $a_2$ deterministically leads to $s_2$. Then, the unique optimal policy is $\pi(s) = a^*$ for every $\rewardParam$ (as $\theta > 0$). This implies that $\I(\theta, d_{\textrm{Rational}}(\theta))= 0$. On the other hand, $d_{\textrm{Boltz}}$ is an injective map (as in Proposition 2), and thus $\I(\theta; d_{\textrm{Boltz}}(\theta)) = \log |\rewardParams|$.  \\
Now, we construct $\{\transProb^{(i)}\}$ such that $a_1$ is optimal for $\theta \leq i-1$, and $a_2$ is optimal for $\theta \geq i$. First, let $a_1$ deterministically lead to $a_2$, thus, leading to reward 0. Next, let $a_2$ lead to $s_1$ with probability $\frac{2}{2i + 1}$. Therefore:
\begin{align*}
    E_{s'\sim P_{s, a_2}^{(i)}} [r_\theta(s, a_2, s')] &= \frac{2}{2i+1} \theta - \frac{2i-1}{2i + 1}
\end{align*}
Then for $\theta \geq i$, we have $E_{s'\sim P_{s, a_2}^{(i)}} [r_\theta(s, a_2, s')]\geq \frac{1}{2i + 1} > 0$ and for $\theta \leq i-1$, we have $E_{s'\sim P_{s, a_2}^{(i)}} [r_\theta(s, a_2, s')]\leq - \frac{1}{2i + 1} < 0$.\\
By construction, in order to compute the optimal policy for every new transition probabilities $\{\transProb^{(i)}\}$, we must know the value of $\theta$, and thus need $\log|\rewardParams|$ bits of information.  
\end{proof}

\end{document}